\documentclass[10pt,twocolumn,letterpaper]{article}

\usepackage{iccv}
\usepackage{times}
\usepackage{epsfig}
\usepackage{graphicx}
\usepackage{amsmath}
\usepackage{amssymb}

\usepackage{mathrsfs}
\usepackage{graphicx}
\usepackage{subfigure}
\newtheorem{theorem}{Theorem}
\newtheorem{proof}{Proof}

\iccvfinalcopy 

\def\iccvPaperID{****} 

\begin{document}

\title{Topology Maintained Structure Encoding}

\author{Qing Fang\\
University of Science and Technology of China\\
Anhui, China\\
{\tt\small fq1208@mail.ustc.edu.cn}
}

\maketitle

\begin{abstract}
Deep learning has been used as a powerful tool for various tasks in computer vision, such as image segmentation, object recognition and data generation.
A key part of end-to-end training is designing the appropriate encoder to extract specific features from the input data.
However, few encoders maintain the topological properties of data, such as connection structures and global contours.
In this paper, we introduce a Voronoi Diagram encoder based on convex set distance (CSVD) and apply it in edge encoding. The boundaries of Voronoi cells is related to detected edges of structures and contours.
The CSVD model improves contour extraction in CNN and structure generation in GAN.
We also show the experimental results and demonstrate that the proposed model has great potentiality in different visual problems where topology information should be involved.

\end{abstract}

\section{Introduction}

\begin{figure}[tbp]
\centering
\subfigure[]{
\begin{minipage}[c]{0.3\linewidth}
  \includegraphics[width=\linewidth]{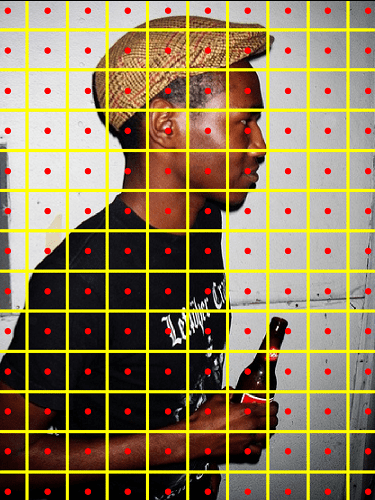}
  \includegraphics[width=\linewidth]{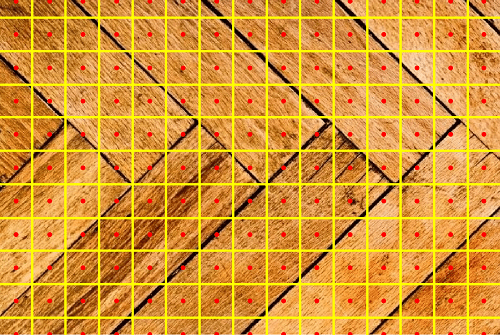}
  \includegraphics[width=\linewidth]{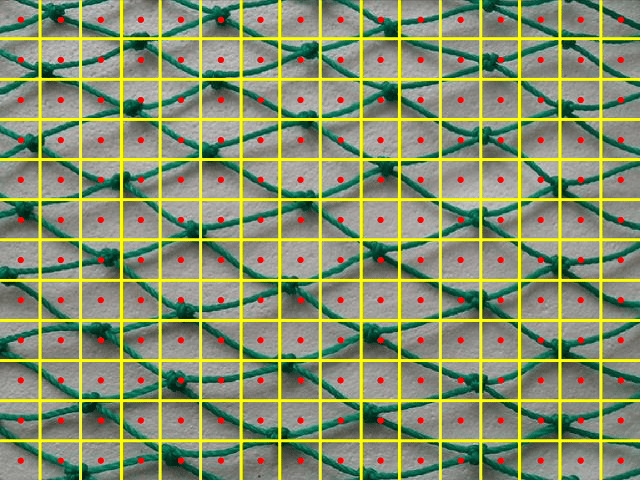}
  \includegraphics[width=\linewidth]{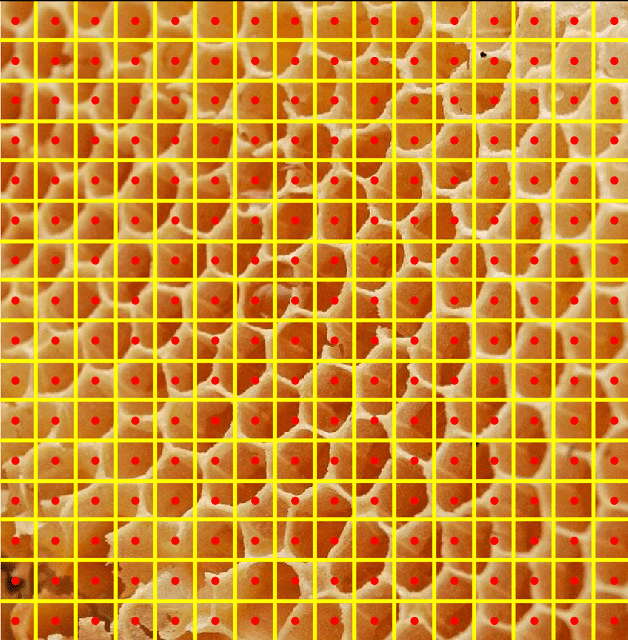}
\end{minipage}
}
\subfigure[]{
\begin{minipage}[c]{0.3\linewidth}
  \includegraphics[width=\linewidth]{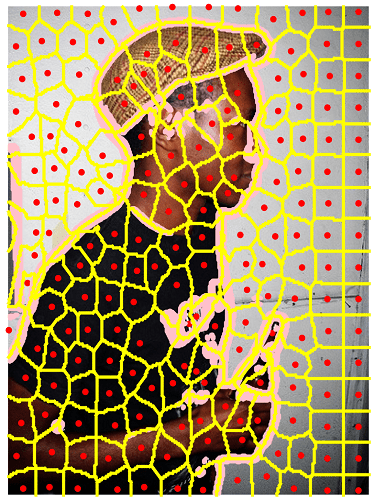}
  \includegraphics[width=\linewidth]{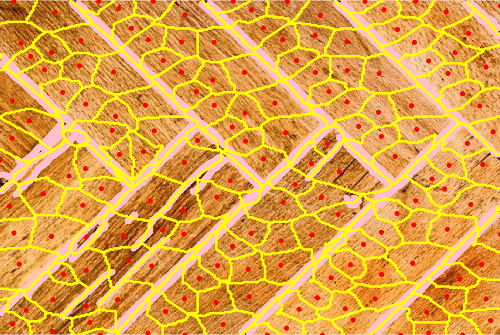}
  \includegraphics[width=\linewidth]{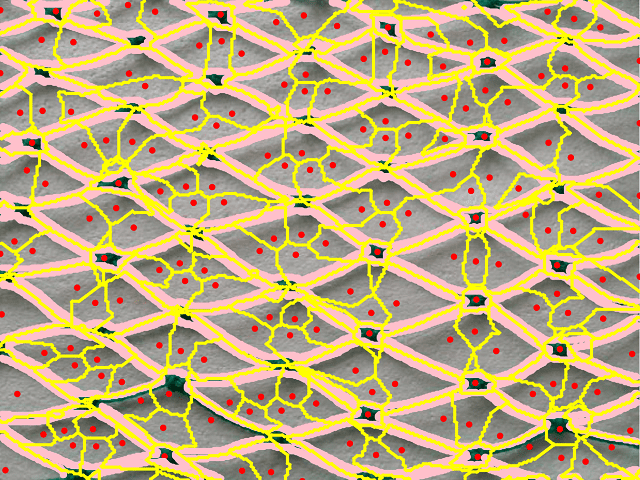}
  \includegraphics[width=\linewidth]{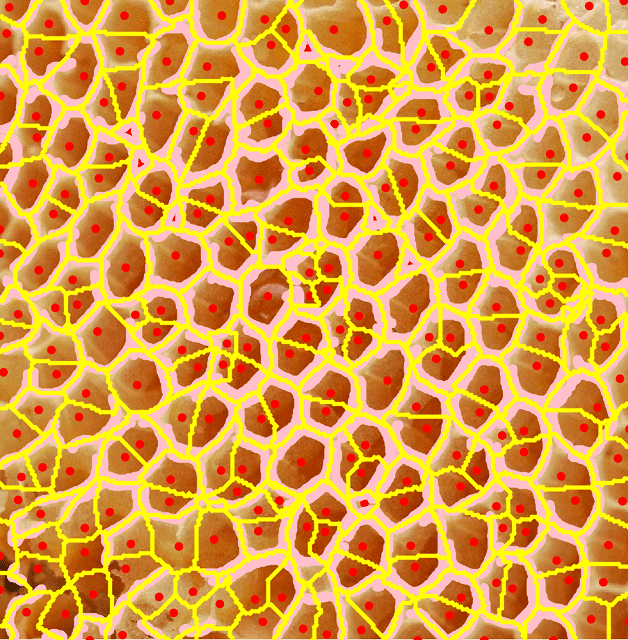}
\end{minipage}
}
\subfigure[]{
\begin{minipage}[c]{0.3\linewidth}
  \includegraphics[width=\linewidth]{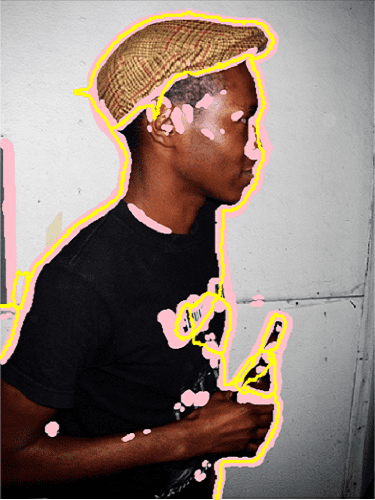}
  \includegraphics[width=\linewidth]{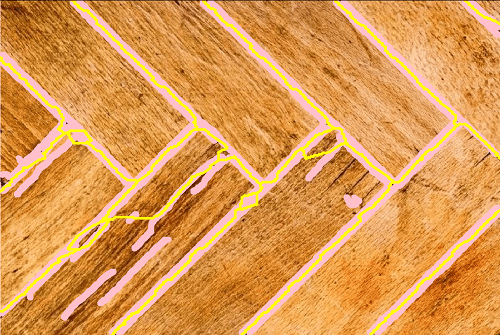}
  \includegraphics[width=\linewidth]{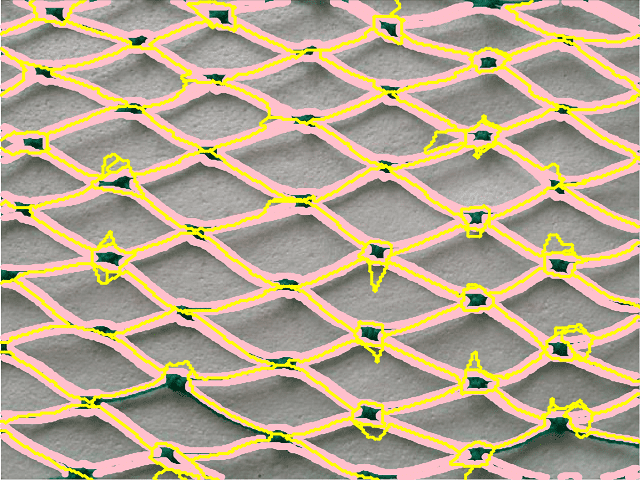}
  \includegraphics[width=\linewidth]{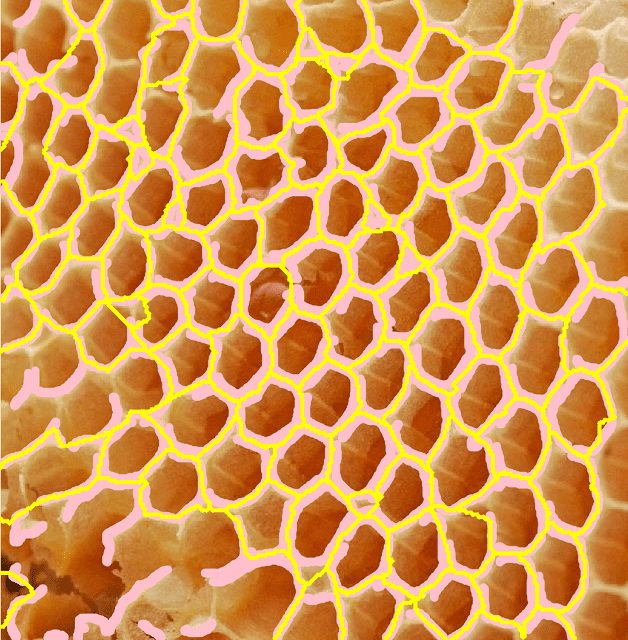}
\end{minipage}
}
\caption{A specific CSVD model for encoding connection edges and structure topology: In the first col (a), the initial state of Voronoi edges is presented by yellow lines; then we detect potential edges of the input image and fit Voronoi edges to boundaries (b); a labeling algorithm is proposed to merge the cells in the same object(hole) and remove redundant edges (c). These examples show that our CSCD encoder performs well in boundary extraction and porous structure fitting.}
\label{fig:structure}
\end{figure}


The development of deep learning has brought about tremendous advances in visual problems. As an effective tool for semi-supervised and unsupervised learning, Generative Adversarial Networks(GANs)~\cite{Goodfellow:2014:GAN:2969033.2969125} recently receives much attention . It is a competitive process between a generator and a discriminator. The generator is used to generate deceptive data while the discriminator is trained to distinguish it with groundtruth. Finally, GANs will be able to synthesize images that can't be recognized~\cite{bau2018gan, DBLP:journals/corr/abs-1710-10196, Wang_2018}. However, most of the state-of-art GAN models are dealing with pictures in pixel level, which may lead to an unnatural appearance caused by disconnected edges. Such topological issues are apparent and significant in some cases, such as handling porous structures both in natural surroundings and artificial object showed in Fig.~\ref{fig:structure}. The connectivity information are expected to be encoded and maintained.


The usage of proper encoders in network model is important for adopting specific information from input data. Convolution kernel is one of the most commonly used encoder, it requires carefully selected filter parameters to achieve satisfying results. The procedure of parameter tuning can be quite laborious and lack of intuition. Some works attempt to design new encoding rules, in which intrinsic characteristics of data are explored to get a better comprehension of features. In neuro-linguistic programming, Cer \etal~\cite{cer2018universal} proposes an semantic encoding model which converts sentences into embedding vectors. In texture recognition, Zhang \etal~\cite{Zhang_2017} integrate CNN layers with dictionary learning. In 3D classification, Riegler \etal~\cite{Riegler_2017_CVPR} apply an octree structure to record distribution of sparse point cloud. When dealing with connectivity structures, it is important to extract their topology features. However, there are few encoders and networks focused on this aspect.


Voronoi diagram (VD) divides a plane into regions based on distance to points in a particular subset. For each point, there is a corresponding region consisting of all points closer to that point than to others. These regions are called Voronoi cells. The common edges between two adjacent cells are called Voronoi edges. VD has been widely applied in different applications, such as Delaunay triangulation for face recognition~\cite{Cheddad_2008} and biological modeling~\cite{Bock2010, Li_2012}.

Inspired by the connectivity of Voronoi edges, we introduce a special convex set distance as Voronoi distance to generate a diagram whose Voronoi edges are forced to coincide with the detected boundaries. Then We devirate an algebraic form of the spcecial convex set distance, which can be embedded in network and calculated by convolutional operation. The network, we call CSVD net, can be seen as a topology information encoder. The encoding process is briefly illustrated in Fig.~\ref{fig:structure}. We indicate that the CSVD encoder is robust to record the different topology among different structures. The representability of CSVD parameters is also tested throngh learning the distribution of parameters by GAN network. The structures generated by GAN  maintain the same statistical characteristics as the input database.


\paragraph{Contributions} The main contributions of our work focus on the following two points:
\begin{enumerate}
\item[1] We introduce a CSVD model to encode topology information of boundaries and convert the encoding process into parameter optimization.
\item[2] A net-based CSVD optimization is proposed to integrate the CSVD model with other networks to encode topological features for different applications.
\end{enumerate}

The rest of the paper is organized as follows. In Section~\ref{sec:related_work}, we briefly introduce the prior research on Voronoi diagram, generative adversarial networks, contours detection and 3D recognition. In Section~\ref{sec:CSVD_Model}, we give a definition of the convex set distance and explain our Voronoi Diagram encoder in detail. In Section~\ref{sec:CSVD_net}, we demonstrate that the CSVD model is able to be embedded into CNN and GAN network and and represent the pattern of structure distribution. The implementation detail is showed in Section~\ref{sec:impl_detail} and the result is evaluated in Section~\ref{sec:result}.


\section{Related Work}
\label{sec:related_work}
\paragraph{Voronoi diagram} Voronoi diagram has been widely applied in computer vision and graphics. The equivalent form of VD is delaunay triangulation which avoids sliver triangles. It is adopted in path planning for automated driving~\cite{Anderson2012ConstraintbasedPA} and face segmentation~\cite{Cheddad_2008}. Kise \etal~\cite{KISE1998370} proposes a direct approximated area Voronoi diagram to analysis and segment page-like images. VD can also be applied into biological structures modeling, such as distribution of cells~\cite{Bock2010} and bone microarchitecture~\cite{Li_2012}.

\paragraph{Generative adversarial networks}
GANs is invented by Ian Goodfellow~\cite{Goodfellow:2014:GAN:2969033.2969125} and has been developed into powerful tools for data generation. Zhu \etal~\cite{Zhu_2016} applies GANs to learn the natural image databases and output reasonable landscape given a few user strokes. A pixel to pixel photo-realistic images synthesis is achieved by conditional GANs~\cite{Wang_2018} in more complicated scenes. Karras \etal~\cite{DBLP:journals/corr/abs-1710-10196} grows both the generator and discriminator progressively in coarse-to-fine manner to generate high quality images with fine details. An analytic framework is proposed to get comprehension of objects and context in real images by~\cite{bau2018gan}, through which an exist unit can be placed in a new surrounding without conflicts. All of these efforts are in pixel level and focused on discrete distributions. So far, all of these efforts are in pixel level and focused on the image discrete distribution. They are unable to maintain the connecting structures of input data. Our net-embedding model can convert the edges distribution to Voronoi diagram parameter, which is continuous and topological invariant. After learning the distribution of VD parameters, the output diagram should have the same topological structure with the input.

\paragraph{Contours detection}
Contours detection has always been a classic topic in image processing. It's widely applied in objects detection, recognition and classification. In edge detection, discontinuities pixels are extracted and a delicate algorithm is designed to trace true edges from these pixels ~\cite{Canny:1986:CAE:11274.11275}. In most visual problems, noise caused be the texture will disturb recognition and a clean outline is necessary. One of the most famous traditional method is snakes~\cite{snakes}, in which an active contour is propagated to track object boundaries. Li \etal~\cite{Li_2016} use an unsupervised learning method to detect the edge in video with the help of optical flow. Data-driven approaches are adopted by~\cite{journals/corr/BertasiusST14, Ganin_2015, conf/cvpr/ShenWWBZ15}. They learns the probability that each pixel belongs to a certain class. Multi-scale convolution layers are included in~\cite{Liu_2018, Xie:2017:HED:3158436.3158453} to take advantage of global image distribution. We will show that the topology information in our CSVD model can help to extract the clean outline. CSVD net has a great potential to be embedded in both supervised and unsupervised learning of contours.

\paragraph{3D object recognition and retrieval}
Traditional convolution layers encounter difficulties when facing 3D data. An additional dimension brings huge costs both in memory and calculation. Different attempts have been made to take advantage of 3D data sparsity. A feature-centric voting scheme is employed to build novel convolutional layers of input point cloud in~\cite{Engelcke_2017}. Multi-view CNNs for object are applied in 3D shape recognition~\cite{Su_2015_ICCV} and improved by introducing multi-resolution filtering in~\cite{Qi_2016}. Octree-based Convolutional Neural Networks appears to encode 3D data adaptively~\cite{Riegler_2017_CVPR, Wang_2017}. The surfaces of objects can be seen as outlines in 2D images. Then the 3D extension of CSVD model can be used to encode shape features. It is more flexible and efficient than voxel-based representation.

\section{CSVD Model}
\label{sec:CSVD_Model}
In this section, we describe our Voronoi diagram model in detail. At first, a convex polygonal distance from~\cite{Ma00bisectorsand} is introduced. We derivate its algebraic form in representation of polygon edge equations. Then a similar distance based on special convex sets is adopted. We indicate that the VD based on the special convex set distance reach the balance of model representability and computational complexity.

\begin{figure}[tbp]
\includegraphics[width=\linewidth]{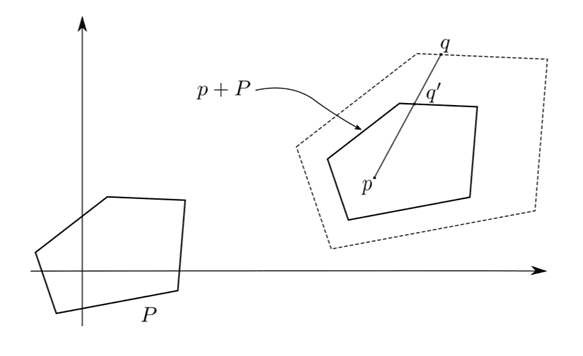}
  \caption{Convex polygonal distance~\cite{Martinez:2018:PVD:3197517.3201343}.}
\label{fig:illustration0}
\end{figure}

\paragraph{Convex polygonal distance}
Given a convex polygon $P$ and a interior point $p$, for any point $q(q \neq p)$ the ray from $p$ to $q$ intersects the polygon $P$ at one point $q'$ [Fig.~\ref{fig:illustration0}], the polygonal distance of $q$ can be defined as:
\begin{equation}
\label{CPDIST}
    d(q|P,p) = \frac{||pq||}{||pq'||}
\end{equation}
At singular point $p$, $d(p|P,p):=0$ to ensure $C^0$ continuity. Eq.~\ref{CPDIST} is hard to calculate. For the convenience of calculation, we derivate an algebraic form of Convex polygonal distance.
\begin{theorem}
\label{THM}
  In convex polygonal distance definition, the straight line passing $p$ and $q$ intersects the edge line $l_i$ of polygon $P$ at point $q_i$ [Fig.~\ref{fig:illustration1} (a)], set equation of $l_i$:
  \begin{equation}
      l_i: \vec{n}_i\cdot (\vec{x}-\vec{x}_p) + b_i=0, ||n_i||=1, b_i>0
  \end{equation}
   then:
  \begin{equation}
    \begin{split}
      d(q|P,p) &= \frac{||pq||}{||pq'||} =\max_i \frac{\vec{pq}}{\vec{pq_i}}\\
               &= \max_i\frac{\vec{n}_i\cdot (\vec{x}_p-\vec{x}_q)}{b_i}
    \end{split}
  \end{equation}
\end{theorem}
\begin{proof}
Because $P$ is convex and $q'$ is the intersection of ray and $P$, line segment $pq'$ is in $P$ and there is no $q_i$ between $p$ and $q'$. Then $\frac{\vec{pq'}}{\vec{pq}}$ is the minimal positive value among all $\frac{\vec{pq_i}}{\vec{pq}}$, which means $\frac{||pq||}{||pq'||} =\max_i=\frac{\vec{pq}}{\vec{pq_i}}$.
Then we have
\begin{equation}
  \frac{\vec{pq}}{\vec{pq_i}} = \frac{\vec{pq}\cdot \vec{n}_i}{\vec{pq_i}\cdot \vec{n}_i}=\frac{(\vec{x}_q-\vec{x}_p)\cdot \vec{n}_i}{(\vec{x}_{q_i}-\vec{x}_p)\cdot \vec{n}_i}
\end{equation}
Because $q_i$ is on $l_i$, \ie $\vec{n}_i\cdot (\vec{x}_{q_i}-\vec{x}_p) + b_i=0$, the convex polygonal distance has the form of a maximum of linear functions:
\begin{equation}
    d(q|P,p) = \max_i\frac{\vec{n}_i\cdot (\vec{x}_p-\vec{x}_q)}{b_i}
\end{equation}
\end{proof}

Assume there are a series of convex polygons $P_k$ and interior points $p_k$, the corresponding VD divides the plane into seperated Voronoi cells. For any plane point q, q belongs to the cell of convex polygon $P_j, j = {\arg\min}_k d(q|P_k,p_k)$. Boundaries of cells represent there are two different convex polygonal distances kept equal. Thus the Voronoi edges are polylines.

\begin{figure}[tbp]
\centering
\subfigure[]{
\begin{minipage}[c]{0.45\linewidth}
\includegraphics[width=\linewidth]{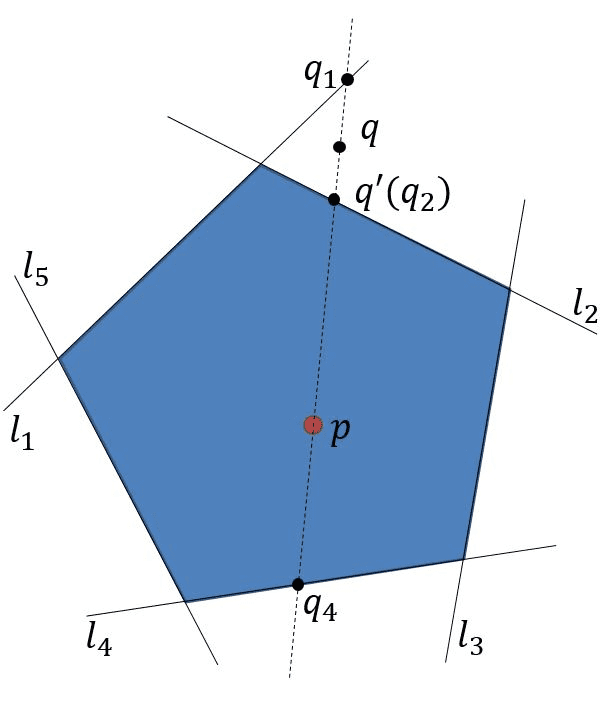}
\end{minipage}
}
\subfigure[]{
\begin{minipage}[c]{0.45\linewidth}
\includegraphics[width=\linewidth]{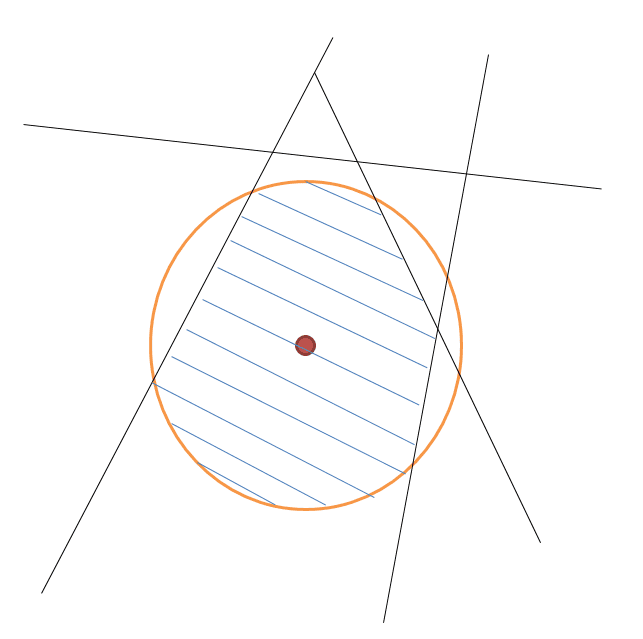}
\end{minipage}
}
  \caption{(a) A convex polygon can be seen as the intersection of the half-planes divided by its edges; (b) a convex set generated by the intersection between several semi-hyperplanes and a circle. }
\label{fig:illustration1}
\end{figure}

\paragraph{Convex set distance} When encoding curve boundaries by convex polygonal distance based VD, we must approximate it with polylines. Icking \etal~\cite{ICKING2001139} generalizes the polygon shape to arbitrary convex set. It is obvious that any ray from a point inner convex set only intersects the set boundary at one point $q'$. Then a similar distance can be defined like Eq.~\ref{CPDIST}. However, the boundary of general convex set is flexible. It brings difficulty in VD calculation and optimization. Fortunately, quadratic curves are sufficient to fit curved boundaries in almost case. Thus only a small change in the shape of convex polygons need to be made. We add a bounded circle centered at $p$ [Fig.~\ref{fig:illustration1} (b)] and it can be proved the boundary of Voronoi cell corresponding to the arc of circle is quadratic. An algebraic form of convex set distance can be derivated similarly:
\begin{equation}
\label{RelaxedMetric}
    d(q|P,p,r) = \max \{\max_i\frac{\vec{n}_i\cdot (\vec{x}_i-\vec{x}_q)}{b_i}, \frac{||\vec{x}_p-\vec{x}_q||}{r}\}
\end{equation}
Here $r$ is the radius of bounded circle.

We mention that adding a bounded circle also relax the restrictions on the edges of convex polygon. In Fig.~\ref{fig:illustration1} (a), convex polygons are represented by the intersection of half-planes. However, not arbitrary half-planes intersection can be bounded, see Fig.~\ref{fig:illustration1} (b). It will lead to ambiguity in convex polygonal distance definition [Eq.~\ref{CPDIST}]. Restriction of keeping the intersection of half-planes bounded is complicated and will bring difficulty in optimization. The modified convex set distance is free-of-restrictions and has more powerful representability.

\section{CSVD Net}
\label{sec:CSVD_net}
In this section, we will explain our topology maintained encoder. We firstly formulate the encoding process as an edge fitting problem by our CSVD model. Then, we show how to merge the CSVD parameters into the convolutional networks and generate the corresponding Voronoi diagram through CNN. At last, we design a target energy and two regularization terms of the Voronoi edges to fit contours and edge-like structures. The CSVD model is flexible. It can be embedded into other network and the optimization can be done through back propagation algorithm.

Edge-like structures and object contours encoding can be divided into a pixel detection phase and an edge fitting phase. The first step is to extract potential pixels which locate at the structure edges or contours. This can be done by Canny operation or other supervised CNN. The second step is fitting the extracted pixel sets with Voronoi edges of CSVD model. We focus on the fitting optimization and explain how to embed CSVD model into CNN net.

\paragraph{Problem}  Given pixel sets $\Omega$ representing potential edge-like structures or object contours in an image, we need to find the parameters of convex sets $\{\vec{x}_{p_k}, \vec{n}_{k,e_k}, b_{k,e_k}, r_k\}$ to make:
\begin{equation}
  d(q|P_j,p_j,r_j) = \min_{k\neq j} d(q|P_k,p_k,r_k), \forall q\in\Omega
\label{EDGEPROP}
\end{equation}
Here $j = arg \min_{k} d(q|P_k,p_k,r_k)$, which means $\Omega$ is contained in the pixel set of Voronoi edge.

\paragraph{CSVD net} The Eq.~\ref{RelaxedMetric} indicates that convex set metric is the maximum of series of linear functions, which is similar to maxpooling in CNN network. Meanwhile, Eq.~\ref{EDGEPROP} is a minpooling operator to get the minimum and second minimum of convex set distances. So the CSVD parameters can be embedded in neural networks and a propagation algorithm can be applied to generate CSVD [Fig.~\ref{fig:Net}].

\begin{figure}[tbp]
\includegraphics[width=\linewidth]{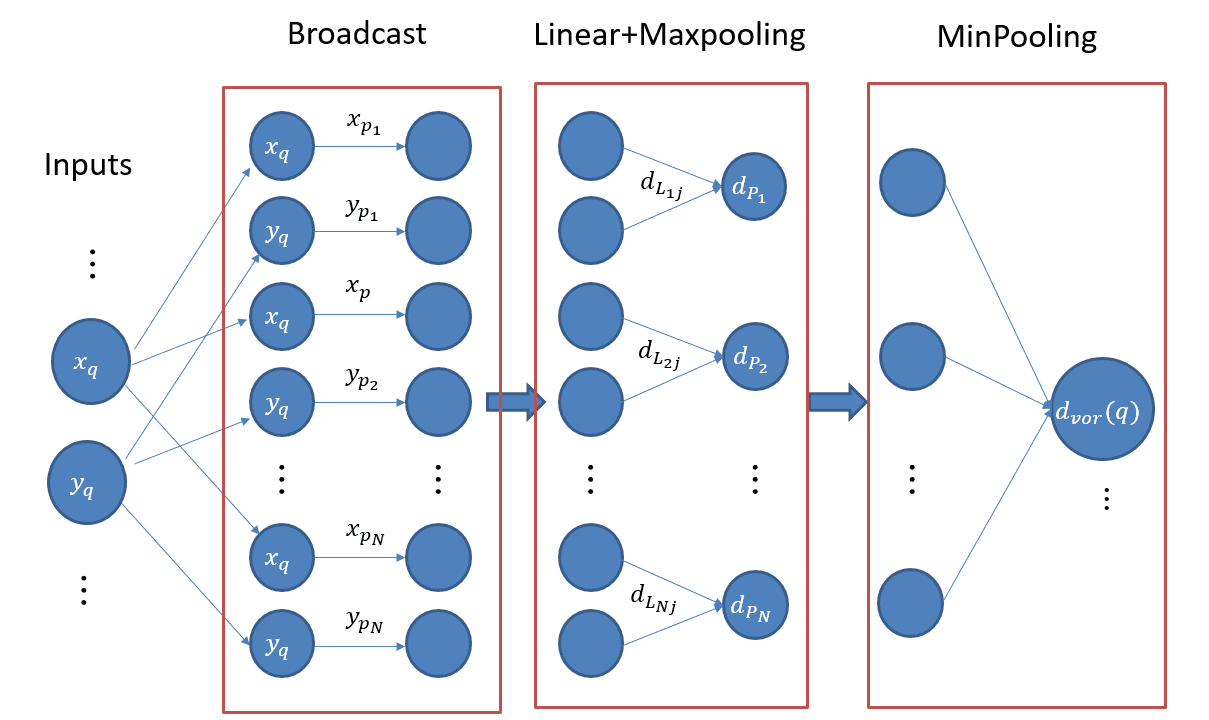}
  \caption{CSVD three layers network: input point $(x_q,y_q)$ is broadcasted to convex sets. For each set, convex set distance is calculated through the maxpooling opration. Then the min distance and voronoi cell $q$ belongs to is obtained by minpooling layer.}
\label{fig:Net}
\end{figure}

\paragraph{Training} When using CSVD to encode topology information, the target energy is to make the extracted boundary edge or connecting structures concide with the edge of voronoi diagram.
\begin{equation}
    E_{tar} = \underset{x\in \Omega}{\mathbf{Means}} (d(x|P_{k_x},p_{k_x}, r_{k,x})- \min_{k \neq k_x} d(x|P_k,p_k,r_k))^2
\label{E1}
\end{equation}
$k_x = arg \min_k d(x|P_k,p_k, r_k) $.
Usually, CSVD parameter is under-constrainted and minimize Eq.~\ref{E1} make $d(x|P_{k_x},p_{k_x}, r_{k,x})$ comes to $0$. We add anchor point to target energy to avoid degeneration.
\begin{equation}
\begin{split}
    E_{tar} =& \underset{x\in \Omega}{\mathbf{Means}}(d(x|P_{k_x},p_{k_x}, r_{k,x})-1)^2\\&+(\min_{k \neq k_x} d(x|P_k,p_k,r_k)-1)^2
\end{split}
\end{equation}
Generally, two regular energy terms are added to guarantee numerical accuracy and Voronoi cell quality.
\begin{equation}
    E_{reg}^1 = ReLU(\epsilon-\min_{i,j}(b_{ij},r_i))
\label{reg1}
\end{equation}
\begin{equation}
    E_{reg}^2 = ReLU(\delta-\min_i\min_{j\neq i}d(p_i|P_j,p_j,r_j))
\label{reg2}
\end{equation}
Here, regularization term Eq.~\ref{reg1} punishes small cells whose scale are under $\epsilon$ and regularization term Eq.~\ref{reg2} avoids adjacent interior points closer than $\delta$. Then CSVD net can be trained by back propagation algorithm to minimize $E_{total}$:
\begin{equation}
    E_{total} = E_{tar}+\lambda_{1}E_{reg}^1+\lambda_{2}E_{reg}^2
\label{energy}
\end{equation}

\section{Implementation details}
\label{sec:impl_detail}

In this section, the algorithm details in contours and porous structures encoding are elaborated. Then we design a GAN architecture to learn the distribution of mesh-like structures.

\begin{figure}[tbp]
\includegraphics[width=\linewidth]{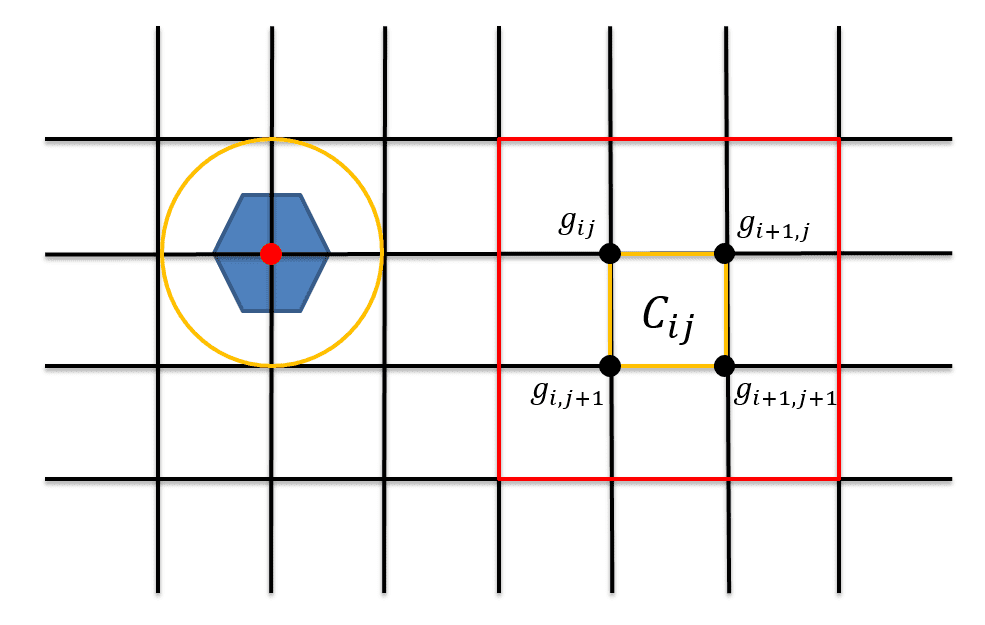}
  \caption{ The left part of the image shows an initial convex set status. The red rectangle on the right represents grid points in the neiborhood $\mathscr{N}_2$ of cell $C_{ij}$.}
\label{fig:illustration2}
\end{figure}

\paragraph{Contours and structures encoding}
The strategy we adopted in fitting contours and structures are scattering and labeling. Because we are not aware of the number of objects or holes in porous structures, the number of convex sets is set to be more than required. These convex sets are placed as seeds in whole image plane and the locations and shapes are optimized to fit the detected edge with Voronoi edges. The a labeling algorithm is designed to merge the Voronoi cells in same objects(holes). We quantify the coincidence of Voronoi edges with detected edges and accumulate values for all pairs of adjacent Voronoi cells. Then we compare the accumulated value with given threshold value and seprate the cells beyond threshold. A flood-fill like method is adopted in merging adjacent cells to generate clean contours. In the initialization phase, We  set CSVD parameters in a delicate manner to take advantage of locality in calculating the minimal convex set distance. We sample the image to generate a square grid in $M\times N$ resolution and attach each grid point $g_{ij}$ with a convex set $S_{ij}$. For each $S_{ij}$, we initialize the location of interior point $p_{ij}$ at $g_{ij}$ and set $P_{ij}$ a regular polygon with $N_e$ number of edges. The radius of bounding circle is $l_u$ and the regular polygon is inscribed in a concentric with radius $\frac{l_u}{2}$ [Fig.~\ref{fig:illustration2}]. To accelerate the generation of CSVD, the minpooling operator can be restricted in a local batch:
\begin{equation}
  d_{minpooling}(q) = \min_{(k,l)\in \mathscr{N}_2(i,j)} d(q|P_{kl},p_{kl},r_{kl})
\end{equation}
Here, $(i,j)$ is the grid cell index where query point $q$ locates in and $\mathscr{N}_2(i,j)$ is the neighborhood region circled by red rectangle in Fig.~\ref{fig:illustration2}. Then Eq.~\ref{EDGEPROP} can be represented:
\begin{equation}
\begin{split}
  d(q|P_{k',l'},p_{k',l'},r_{k',l'}) = \min_{\substack{(k,l)\in \mathscr{N}_2(i,j),\\(k,l)\neq(k',l')}} d(P_{kl},p_{kl},r_{kl})
\end{split}
\label{MINILOCAL}
\end{equation}
Here, $(k',l')=arg \min_{(k,l)\in \mathscr{N}_2(i,j)} d(q|P_{kl},p_{kl},r_{kl})$.

\begin{figure}[tbp]
\includegraphics[width=\linewidth]{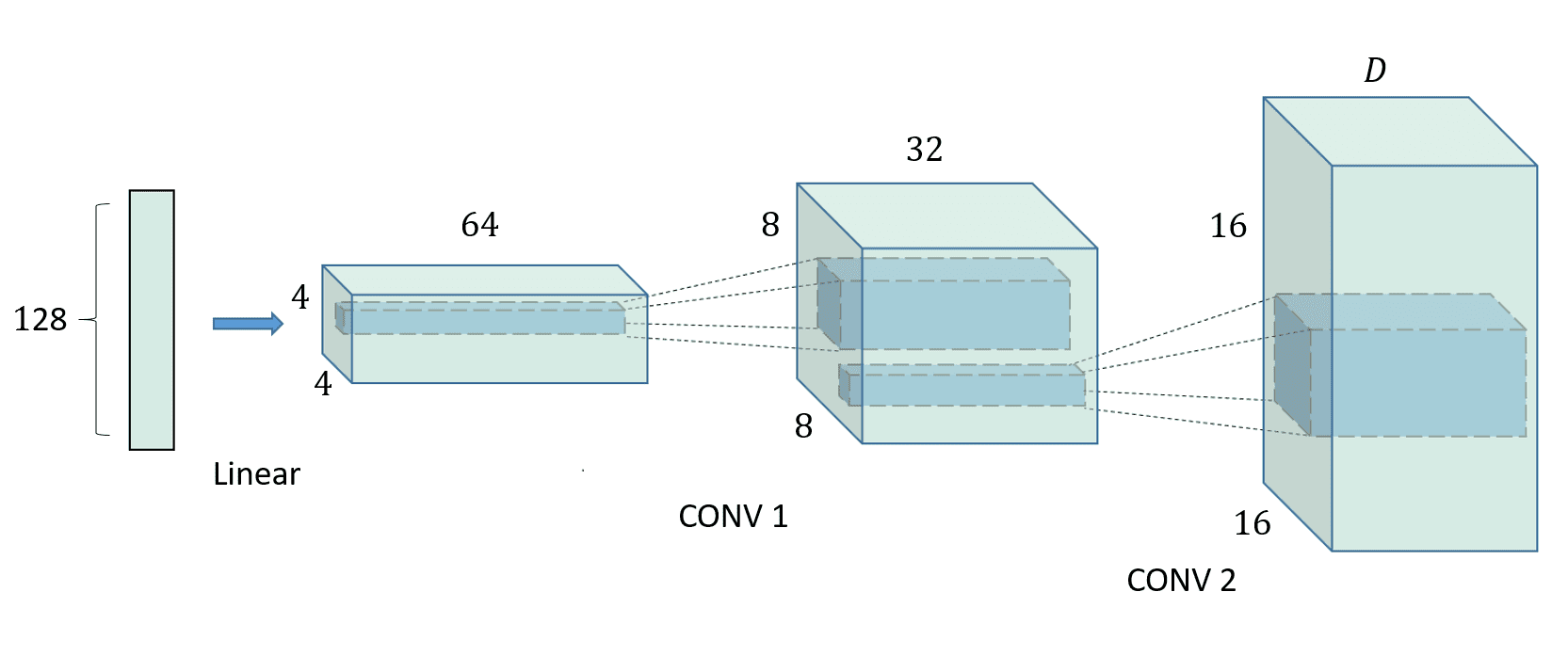}
  \caption{ A 128 dimensional uniform distribution Z is fed into linear layer to output a convolutional representation with many feature maps. Two  convolutions convert high level features into CSVD parameters space.}
\label{fig:generator}
\end{figure}

\paragraph{Mesh-like structures learning}
Through encoding phase, the structure information is recorded by the locations of interior points and shapes of convex sets. Then the feature of structures can be extracted by learning corresponding distribution of CSVD parameters. In this way, structure generation can be converted to parameters generation. Thus we design a GAN to learn the distribution of mesh-like structures parameters. The parameter from encoding phase is a $M\times N\times D$ tensor. Here, $M\times N$ is the grid shape($16\times 16$ default) and $D$ is the dimension of each convex set parameter and labels. We set the dim of latent space $dim(z)=128$. Thus the generator is a three layer net consisting of a linear layer and two convolutional layers [Fig.~\ref{fig:generator}]. The discriminator is set to the reverse of the generator and change the output dim from $128$ to $1$. Then We add the gradient penalty~\cite{gulrajani2017improved} to in training phase.

\section{Results}
\label{sec:result}

We test our CSVD encoding ability in PASCAL VOC2012 images database~\cite{pascal-voc-2012} and Describable Textures Dataset~\cite{cimpoi14describing}. We design a edge detection CNN network for contours extraction. The convolutional kernel of the CNN net is initialized with traditional gradient filters, such as sobel and LoG filters. Then we cascade the detection network with our CSVD net and train the parameters of detection and CSVD net together to minimize the energy Eq.~\ref{energy}. In our experiments, the input image is scaled into a unit square by divide its shape and the weights of regularization terms are set as $\lambda_1=\lambda_2=1$. The results is showed in Fig.~\ref{fig:result1}. From the picture, we can see that training detection and CSVD net together will help to remove redundant edges and generate a clean contour. In cells-like structures encoding, the Voronoi edges between cells in different labels align with structure edges tightly.

To test the representability of CSVD model in cells-like structures, we first create $2000$ images with cells-like structures in different topology. This can be done by sampling points in whole plane non-uniformly and generating Voronoi diagram of these seeds. In our experiments, the scale of voronoi cells decreases from the top left corner to the bottom right corner. We show the GAN outcome in Fig.~\ref{fig:result2}. It is apparent that the output parameters satisfy the distribution of database. However, we can't promise the boundaries between cells of different labels to be straight lines. This is because there is just a simple classification in the discriminator. We only focus on the statistical behavior of parameters. The local property is ignored.

\begin{figure}[tbp]
\centering
\subfigure[]{
\begin{minipage}[c]{0.3\linewidth}
 \includegraphics[width=\linewidth]{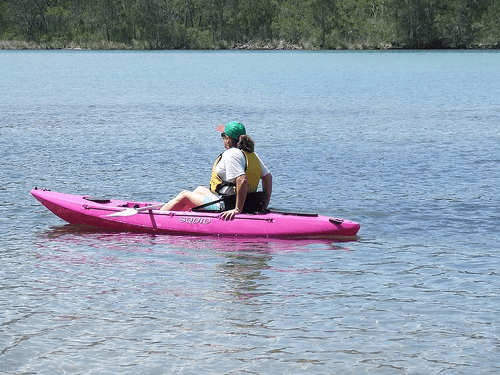}
 \includegraphics[width=\linewidth]{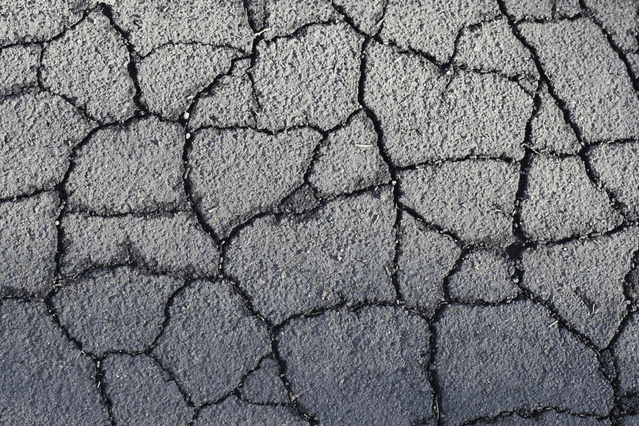}
 \includegraphics[width=\linewidth]{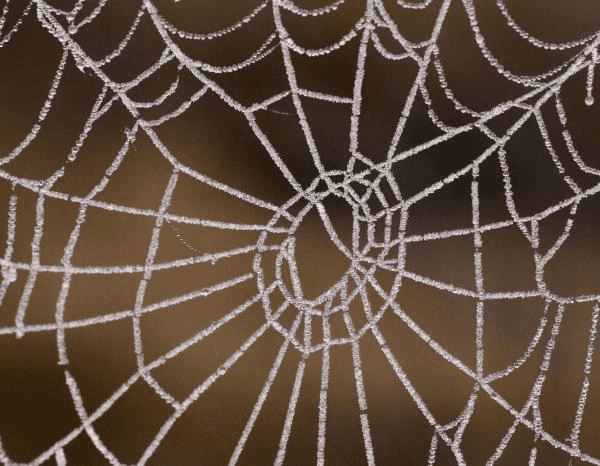}
 \includegraphics[width=\linewidth]{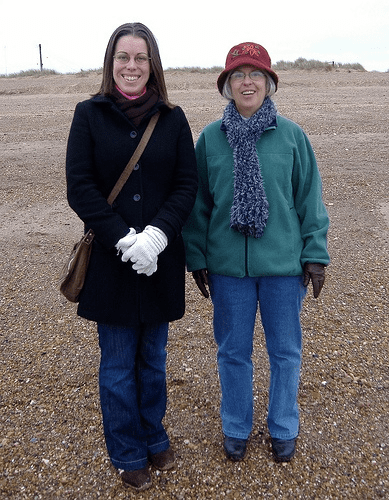}
\end{minipage}
}
\subfigure[]{
\begin{minipage}[c]{0.3\linewidth}
  \includegraphics[width=\linewidth]{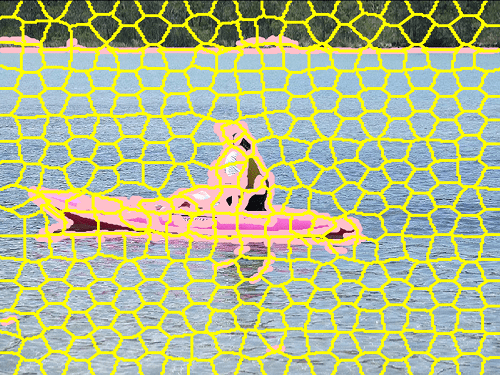}
  \includegraphics[width=\linewidth]{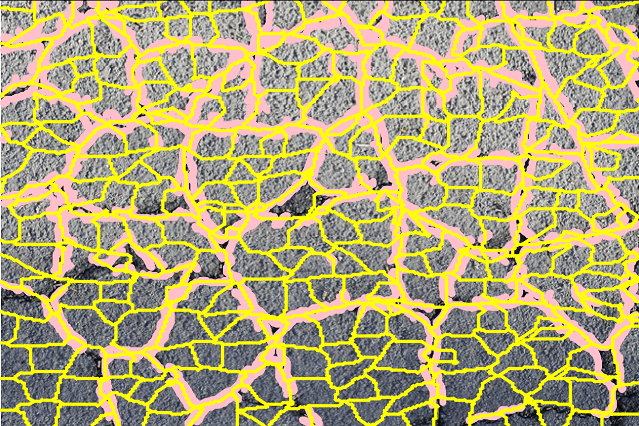}
  \includegraphics[width=\linewidth]{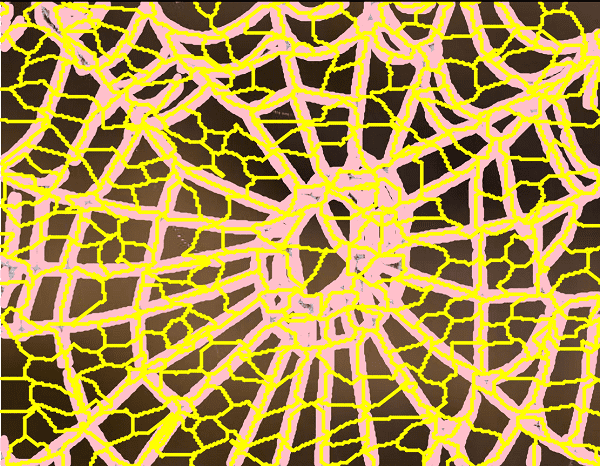}
  \includegraphics[width=\linewidth]{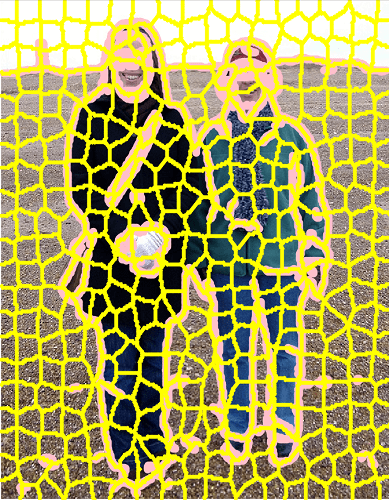}
\end{minipage}
}
\subfigure[]{
\begin{minipage}[c]{0.3\linewidth}
  \includegraphics[width=\linewidth]{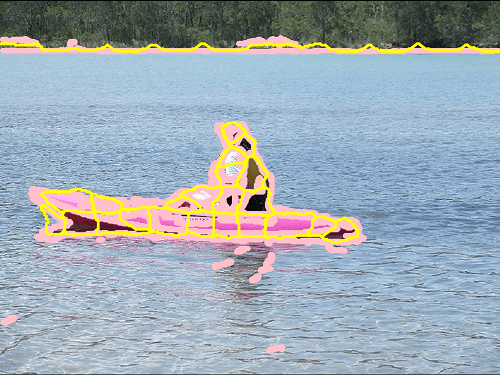}
  \includegraphics[width=\linewidth]{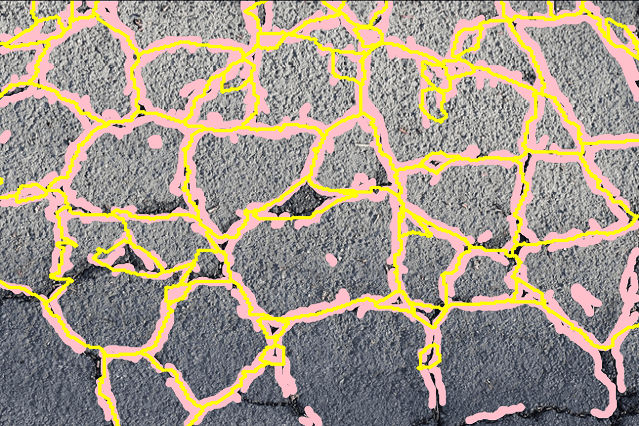}
  \includegraphics[width=\linewidth]{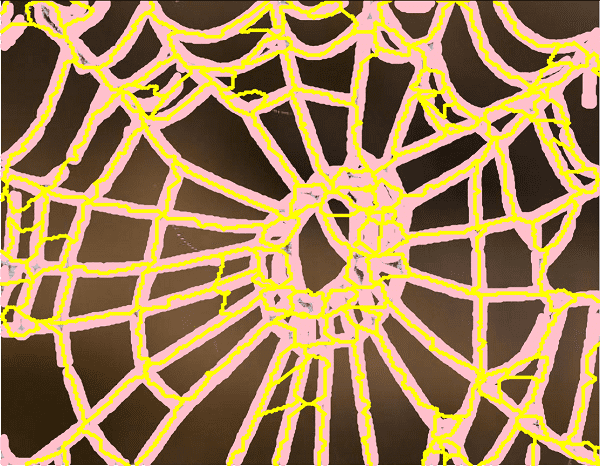}
  \includegraphics[width=\linewidth]{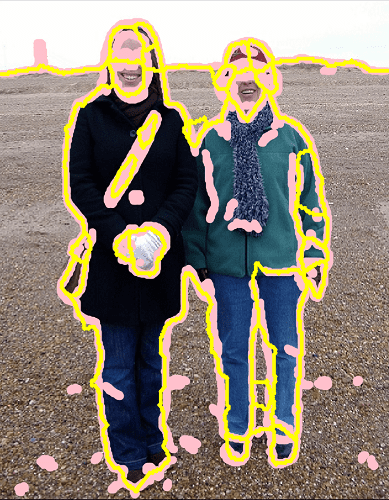}
\end{minipage}
}
  \caption{More results are tested on different scenes and porous structures. }
\label{fig:result1}
\end{figure}

\begin{figure}[tbp]
\centering
\subfigure[]{
\begin{minipage}[c]{0.3\linewidth}
\includegraphics[width=\linewidth]{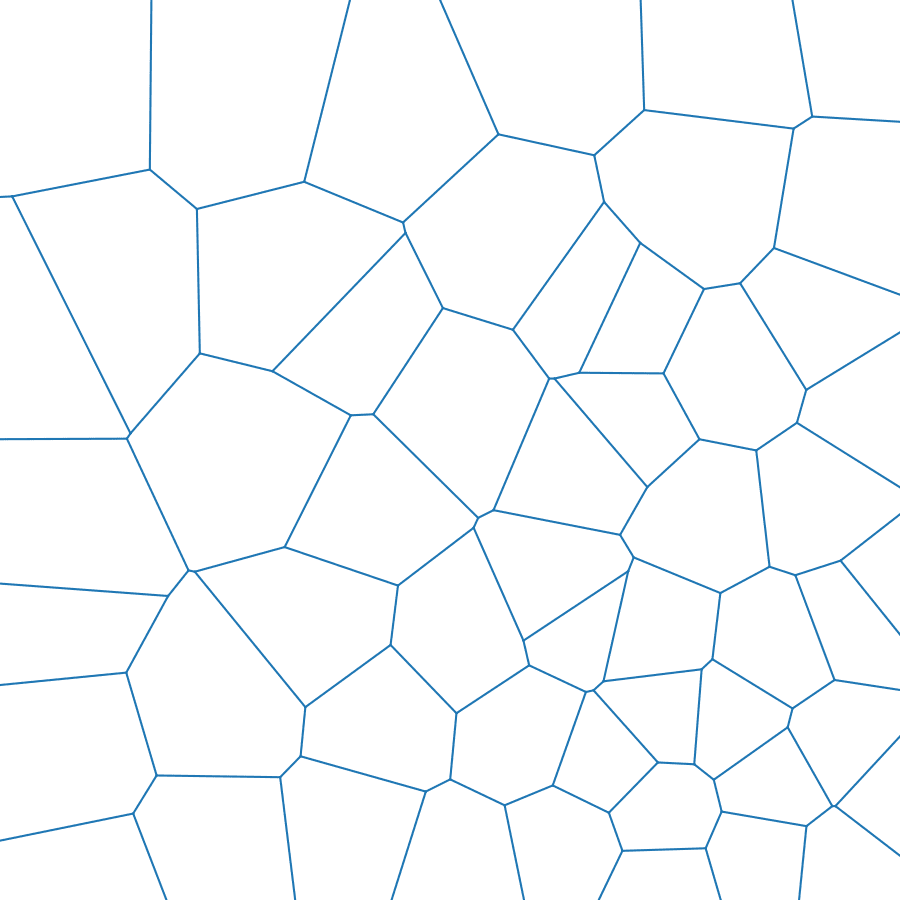}
\includegraphics[width=\linewidth]{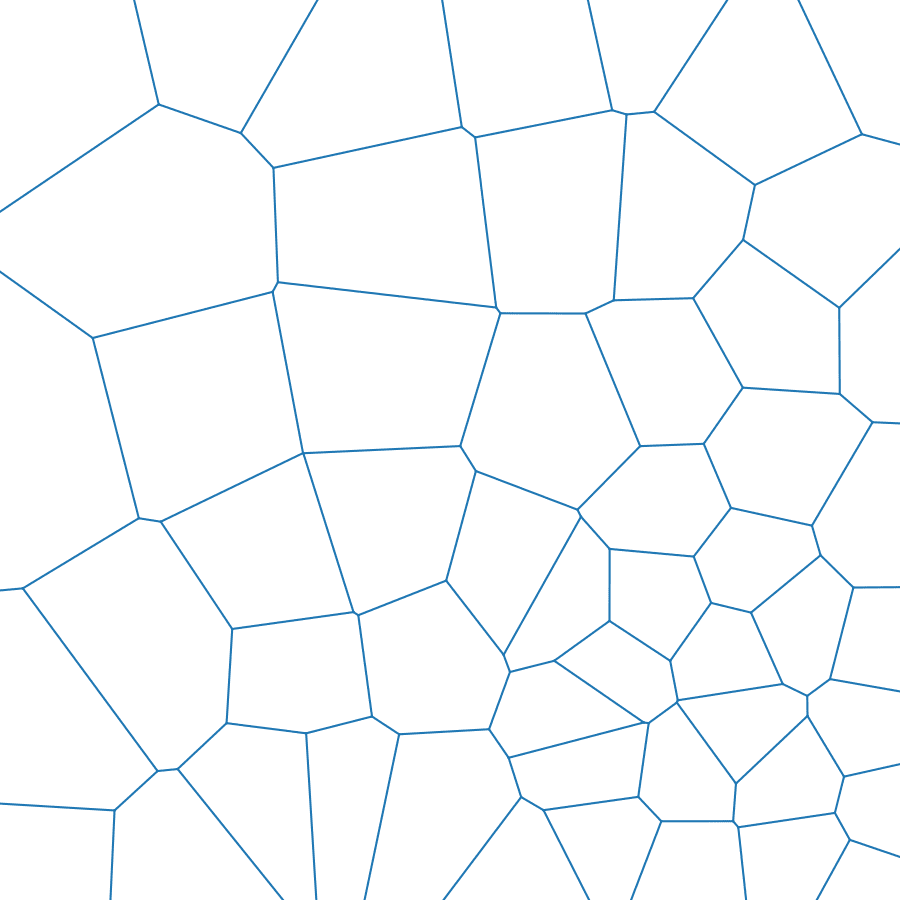}
\includegraphics[width=\linewidth]{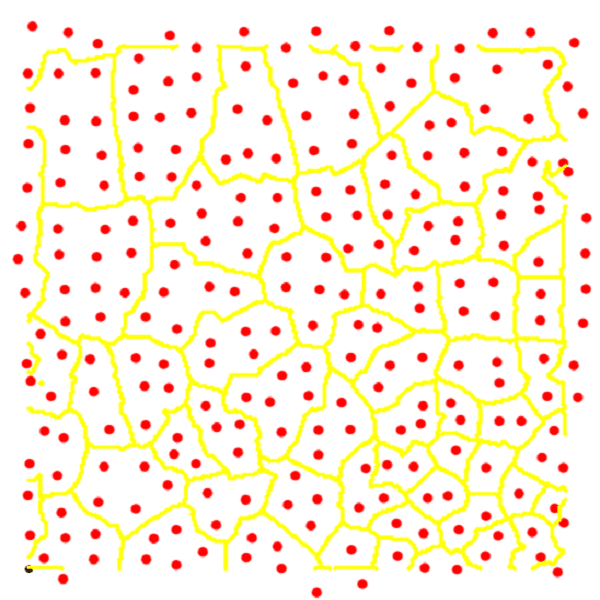}
\end{minipage}
}
\subfigure[]{
\begin{minipage}[c]{0.3\linewidth}
  \includegraphics[width=\linewidth]{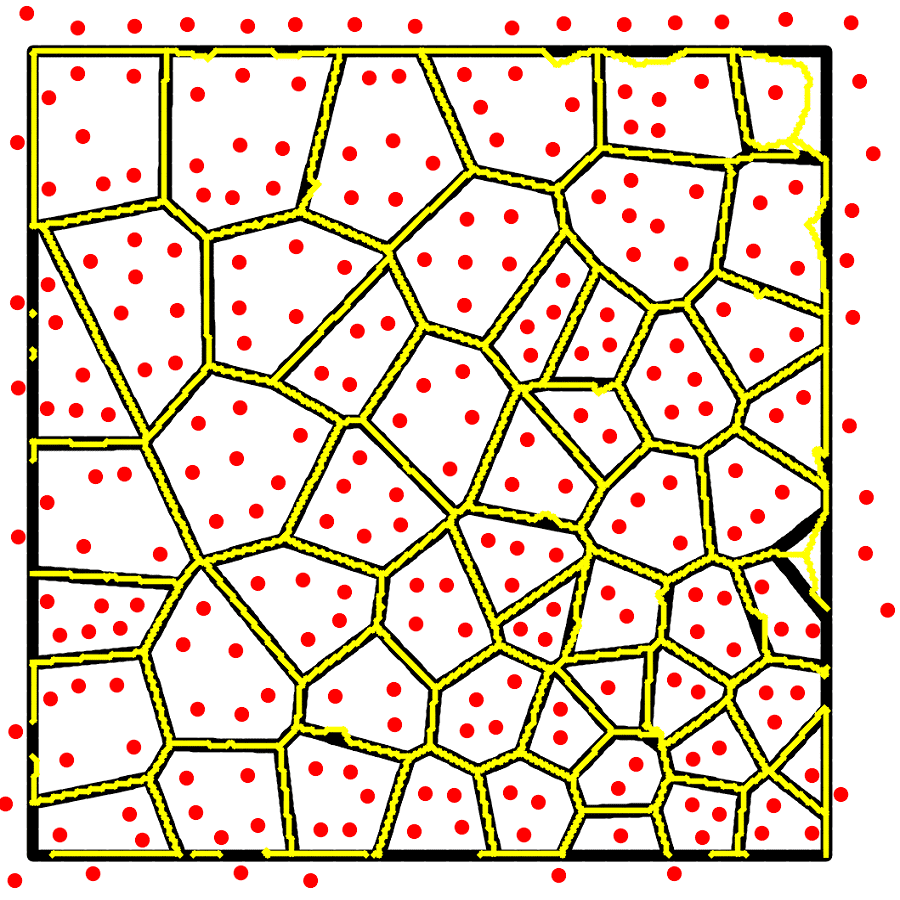}
  \includegraphics[width=\linewidth]{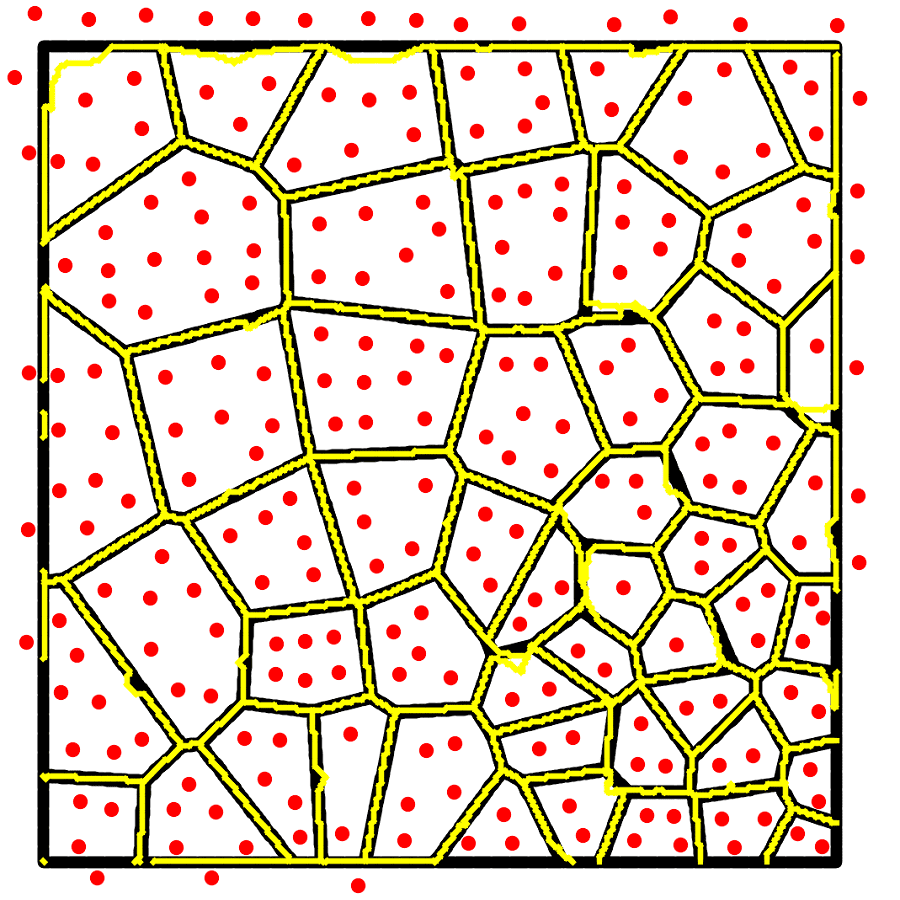}
  \includegraphics[width=\linewidth]{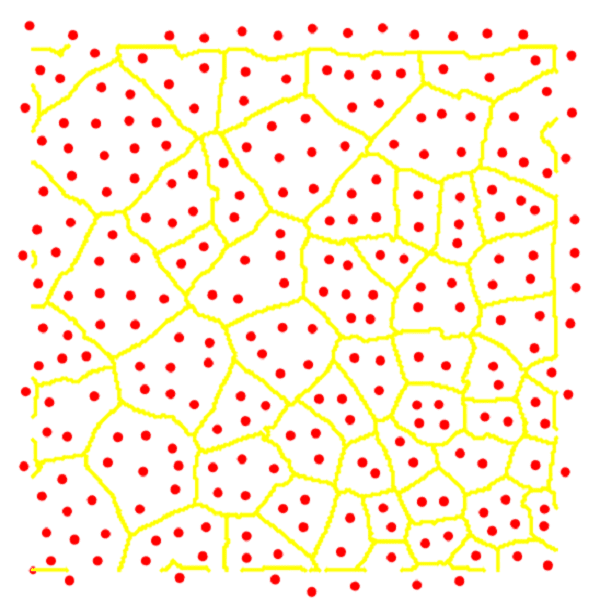}
\end{minipage}
}
\subfigure[]{
\begin{minipage}[c]{0.3\linewidth}
  \includegraphics[width=\linewidth]{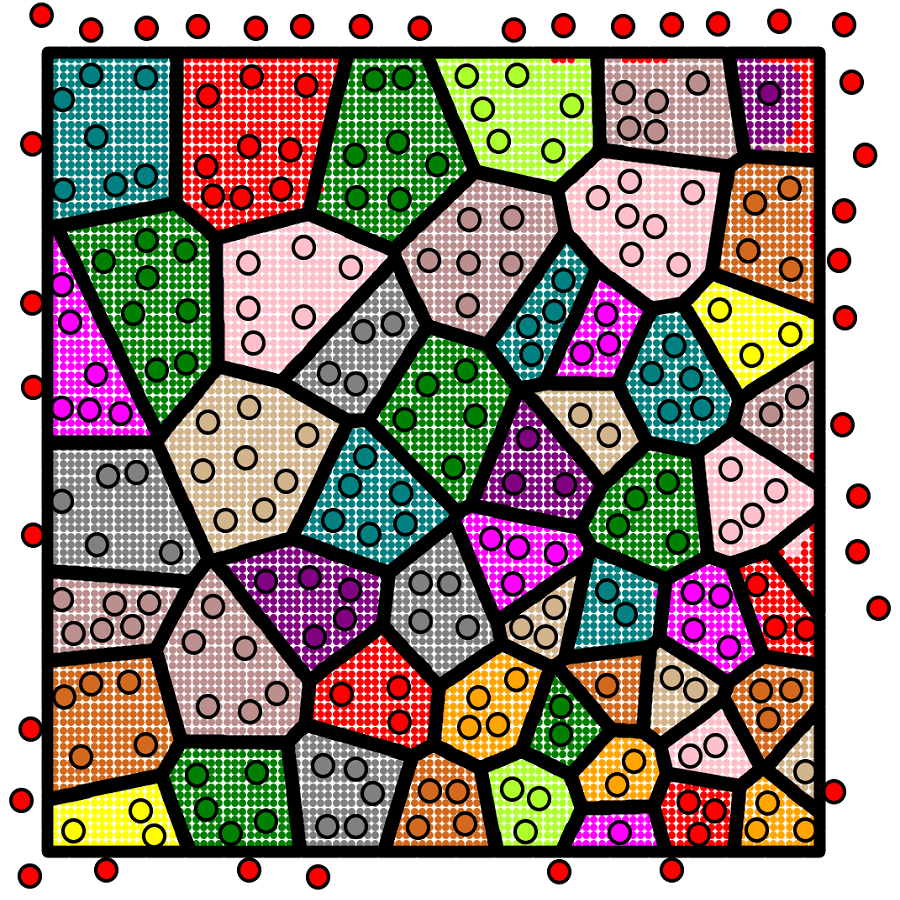}
  \includegraphics[width=\linewidth]{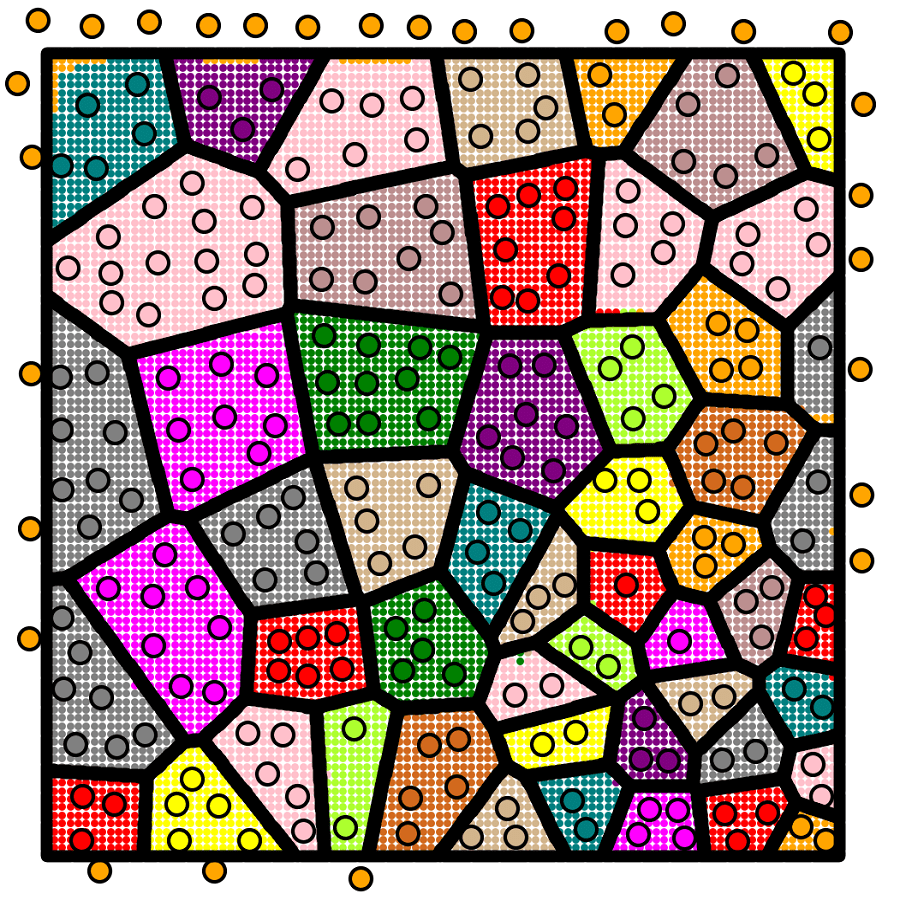}
  \includegraphics[width=\linewidth]{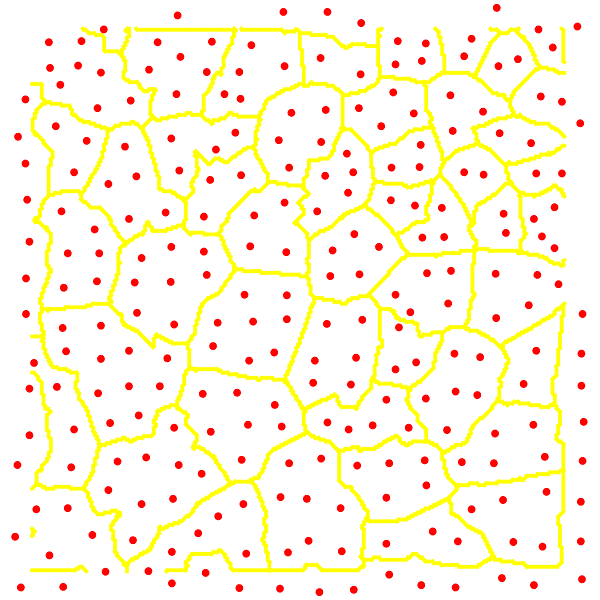}
\end{minipage}
}
  \caption{The first two rows show the fitting results of two artificial images in different structure distribution. In the middle col the coincidence degree of structure edges and Voronoi edges is presented; We color the Voronoi cells of different colors for different labels in the third column. The last row shows severial examples generated by GAN. }
\label{fig:result2}
\end{figure}


\section{Conclusion and Future works}
\label{sec:ConclusionAndFutureWork}
In this paper, we introduce a special Voronoi diagram based on convex set distance and first apply it in encoding contours and structures. By optimizing the parameters of convex set distance, the Voronoi edge is fitted to the potential contours. The CSVD model is embedded into convolutional layers and can be trained with other learning tasks. The encoding process generate a uniform parameterization of objects(structures) and the contours(connectivity edges) are determined by labels of Voronoi cells. The geometric interpretation of CSVD model avoid can be applied in helps to avoid over-fitting, which appears commonly in deep network.

We indicate that our CSVD encoder has the potential in many other aplications. In image recognition, CSVD encoder can track the object boundary and provide additional topology information. In 3D object classification and recognition, voxel-based convolutional operation is difficult due to the high computation and memory. Then 3D extension of CSVD encoder may take advantage of the sparsity of contour and parameterize objects volume uniformly. The CSVD model can also encode the distribution of structures with holes, such as tooth, bone in biology and porous structures in 3D printing and architecture.

{\small
\bibliographystyle{ieee}
\bibliography{egbib}
}

\end{document}